\documentclass[sigconf]{acmart}

\usepackage{algorithm}
\usepackage{algorithmic}
\usepackage{amsmath}
\usepackage{multirow}
\usepackage{multicol}
\usepackage{balance}
\usepackage{mathtools}
\usepackage{appendix}

\AtBeginDocument{%
  }

\copyrightyear{2024}
\acmYear{2024}
\setcopyright{acmlicensed}\acmConference[RecSys '24]{18th ACM Conference on Recommender Systems}{October 14--18, 2024}{Bari, Italy}
\acmBooktitle{18th ACM Conference on Recommender Systems (RecSys '24), October 14--18, 2024, Bari, Italy}
\acmDOI{10.1145/3640457.3688128}
\acmISBN{979-8-4007-0505-2/24/10}

\acmConference[Recsys '24]{18th ACM Conference on Recommender Systems}{October 14–18, 2024}{Bari, Italy}
\acmISBN{978-1-4503-XXXX-X/18/06}

\begin{document}


\title{RPAF: A Reinforcement Prediction-Allocation Framework for Cache Allocation in Large-Scale Recommender Systems}

\author{Shuo Su}
\authornote{Both authors contributed equally to this research.}
\affiliation{%
  \institution{Kuaishou Technology}
  \city{Beijing}
  \country{China}
}
\email{sushuo@kuaishou.com}

\author{Xiaoshuang Chen}
\authornotemark[1]
\orcid{0000-0003-1267-1680}
\affiliation{%
  \institution{Kuaishou Technology}
  \city{Beijing}
  \country{China}
}
\email{chenxiaoshuang@kuaishou.com}

\author{Yao Wang}
\affiliation{%
  \institution{Kuaishou Technology}
  \city{Beijing}
  \country{China}}
\email{wangyiyan@kuaishou.com}

\author{Yulin Wu}
\affiliation{%
  \institution{Kuaishou Technology}
  \city{Beijing}
  \country{China}
}
\email{wuyulin@kuaishou.com}

\author{Ziqiang Zhang}
\affiliation{%
  \institution{Tsinghua University}
  \city{Beijing}
  \country{China}
}
\email{zzq23@mails.tsinghua.edu.cn}

\author{Kaiqiao Zhan}
\affiliation{%
  \institution{Kuaishou Technology}
  \city{Beijing}
  \country{China}
}
\email{zhankaiqiao@kuaishou.com}

\author{Ben Wang}\authornote{Corresponding author.}
\affiliation{%
  \institution{Kuaishou Technology}
  \city{Beijing}
  \country{China}
}
\email{wangben@kuaishou.com}

\author{Kun Gai}
\affiliation{%
  \institution{Unaffiliated}
  \city{Beijing}
  \country{China}
}
\email{gai.kun@qq.com}
\renewcommand{\shortauthors}{Shuo su et al.}

\begin{abstract}
 Modern recommender systems are built upon computation-intensive infrastructure, and it is challenging to perform real-time computation for each request, especially in peak periods, due to the limited computational resources. Recommending by user-wise result caches is widely used when the system cannot afford a real-time recommendation. However, it is challenging to allocate real-time and cached recommendations to maximize the users' overall engagement. This paper shows two key challenges to cache allocation, i.e., the value-strategy dependency and the streaming allocation. Then, we propose a reinforcement prediction-allocation framework (RPAF) to address these issues. RPAF is a reinforcement-learning-based two-stage framework containing prediction and allocation stages. The prediction stage estimates the values of the cache choices considering the value-strategy dependency, and the allocation stage determines the cache choices for each individual request while satisfying the global budget constraint. We show that the challenge of training RPAF includes globality and the strictness of budget constraints, and a relaxed local allocator (RLA) is proposed to address this issue. Moreover, a PoolRank algorithm is used in the allocation stage to deal with the streaming allocation problem. Experiments show that RPAF significantly improves users' engagement under computational budget constraints.
\end{abstract}






\maketitle

\section{Introduction}
Modern recommender systems are built upon computation-intensive infrastructure \cite{jiang2020dcaf,liu2017cascade,johnson2019billion,wang2020cold}. Millions or even billions of users visit the system, leading to huge traffic \cite{jiang2020dcaf,zhou2023rl}. Moreover, the recommender system's traffic is usually highly different between peak and off-peak periods, and the computational burden during peak periods is several times that of off-peak periods \cite{chen2024cache}, which leads to a significant challenge of computational burden in peak periods.

The result cache \cite{jiang2020dcaf,chen2024cache} is widely used in large-scale recommender systems to address this issue. As shown in Figure \ref{fig:cache}, when the system receives a user's request, it first provides a real-time recommendation, returning a group of items more than required, of which the rest are put into a result cache. When the user's next request comes, the system will perform a real-time recommendation again if the traffic does not exceed the system's affordability. Otherwise, it will recommend items directly from the result cache. In peak periods, the number of cached recommendations is comparable to that of real-time recommendations, and the cache largely mitigates the recommender systems' computational burden \cite{chen2024cache,zhou2023rl,jiang2020dcaf}.

\begin{figure}
    \centering
    \includegraphics[width=\columnwidth,trim=0 25 0 0, clip]{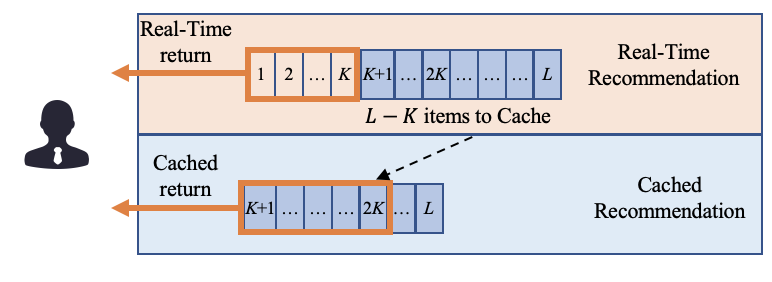}
    \caption{Recommendation with a result cache.}
    \label{fig:cache}
\end{figure}

\begin{figure}
    \centering
    \includegraphics[width=\columnwidth,trim=0 0 0 0, clip]
    {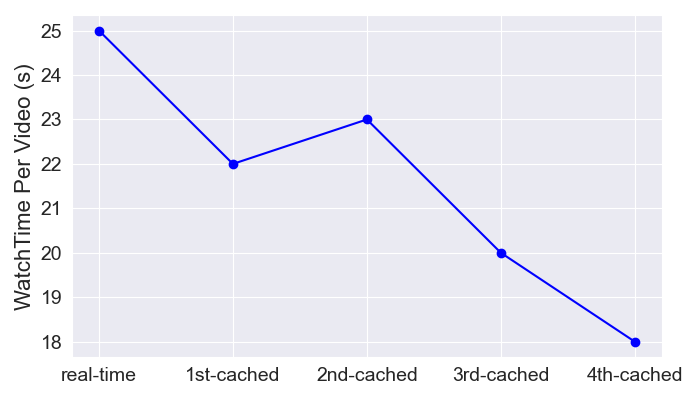}
    \vspace{-5mm}
    \caption{WatchTime decreases when continuously receiving cached recommendations.}
    \label{fig:continuous-watch-time}
    \vspace{-5mm}
\end{figure}

The cached results are typically suboptimal compared to real-time recommendations due to the lack of real-time computation, and allocating the real-time and cached recommendations for each user request is critical to maximizing the overall user experiences under computational budget constraints. Research on computational resource allocation problems \cite{jiang2020dcaf,yang2021computation} models the allocation problem as a constrained optimization problem. However, the cache allocation problem cannot be regarded as a special case of the computational resource allocation problem due to the following challenges:
\begin{itemize}
    \item \textbf{Value-Strategy Dependency}. Existing computational resource allocation methods assume that the requests of different time periods are independent and that the values of the computational resources are independent of the allocation strategy. However, such assumptions do not hold in the cache allocation problem. On the one hand, the size of the result cache is finite. Figure \ref{fig:continuous-watch-time} shows the users' average watch time when continuously receiving recommendations from the cache in a short video platform. If the system continuously provides cached recommendations to the same user, the result cache will soon be exhausted, and the user experience will quickly deteriorate. On the other hand, since the system's choice of whether to use the cache will affect not only the users' feedback on the current request but also the users' future actions, the values of the current cache choices also depend on the future cache allocation strategies.
    \item \textbf{Streaming Allocation}. Existing computational resource allocation approaches solve the allocation problem for a batch of requests in each time period. However, requests in online recommender systems arrive in a streaming manner, and the system needs to determine the cache choice when each individual request arrives while satisfying the global computational budget constraint.
\end{itemize}

To this end, we introduce a Reinforcement Prediction-Allocation Framework (RPAF). RPAF is a two-stage method with a prediction stage and an allocation stage. The prediction stage uses reinforcement learning to estimate the values under different cache choices considering value-strategy dependency, and the allocation stage uses the estimated values to perform the streaming allocation.

A key challenge in the RPAF is to deal with the budget constraint, which is global and strict, in the reinforcement learning model of the prediction stage. To address the globality issue, we introduce a relaxed local allocator (RLA). RLA is a policy function that can be optimized locally and continuously for each request. By using RLA, the constrained reinforcement learning problem in the prediction stage is transformed into a computationally tractable form. Based on RLA, the prediction stage provides an estimation of different cache choices considering the global budget constraint. To address the strictness issue, we provide a PoolRank algorithm to provide the allocation strategy in a streaming manner, satisfying the budget constraint strictly at each time step.

In summary, our contributions can be summarized as follows:
\begin{itemize}
\item We propose RPAF to perform the cache allocation with value-strategy dependency and streaming allocation.
\item We introduce RLA to estimate the value of the cache choices considering the impacts of the allocation strategy.
\item We provide a PoolRank algorithm to allocate the real-time and cached recommendations in a streaming manner.
\item We validate the effectiveness of RPAF through offline experiments and online A/B tests. RPAF has been deployed online in real-world applications, bringing a notable improvement.
\end{itemize}

\section{Related Work}
\subsection{Computation Resource Allocation in Recommender Systems}
Recommender systems have been widely researched in recent years, but most studies focus on improving business revenue or user experience under sufficient computational resources \cite{zhu2018learning,pi2020search,wang2020cold}. 
Some studies focus on the computational resource allocation problem in recommender systems, such as DCAF \cite{jiang2020dcaf}, CRAS \cite{yang2021computation}, RL-MPCA \cite{zhou2023rl}. DCAF and CRAS consider the computational resource allocation problem as a constrained optimization problem solved by
linear programming. RL-MPCA discusses the computational resource allocation problem in multi-phase recommender systems. However, existing methods do not consider the value-strategy dependency in the cache allocation problem studied in this paper.

\subsection{RL in Recommender Systems}
Reinforcement learning (RL) in recommender systems has gained significant attention due to its ability to optimize user long-term satisfaction \cite{cai2023two,afsar2022reinforcement,chen2024cache}. Value-based approaches estimate the user's satisfaction with recommended items from the available candidate set, and then select the one predicted to yield the highest satisfaction \cite{chen2018stabilizing,zhao2018recommendations}. Policy-based methods directly learn the recommendation policy and optimize it to increase user satisfaction \cite{chen2019large, ma2020off,chen2024cache}. 
RL-based models can consider the temporal dependency of the cache allocation problem, but it is challenging to consider the computational budget constraints in RL models. RL-MPCA \cite{zhou2023rl} uses deep Q-Network (DQN) with a dual variable to model the constraints, but we will show in this paper that DQN-like value-based methods face challenges when modeling the cache allocation problem with global and strict budget constraint. In contrast, we propose the RPAF with RLA to tackle the abovementioned challenges.

\subsection{CMDP}
A common approach to solving constrained Markov Decision Processes (CMDP) is the Lagrangian relaxation method\cite{tessler2018reward,chow2018risk,dalal2018safe,garcia2015comprehensive,liu2021policy}. The constrained optimization problem is reduced to unconstrained by augmenting the objective with a sum of the constraint functions weighted by their corresponding Lagrangian multipliers. Then, the Lagrangian multipliers are updated in the dual problem to satisfy the constraints. However, existing methods only consider the local constraints under each request and, therefore, cannot be directly applied in the cache allocation problem where the budget constraints are global, based on all requests at that moment. Moreover, existing methods meet the constraint in the sense of expectation, while the budget constraint in our scenario is strict. In contrast, the RPAF considers globality and the strictness of budget constraints.
 
\begin{figure}[t]
    \centering
    \includegraphics[width=1.0\linewidth,trim=0 90 0 0, clip]{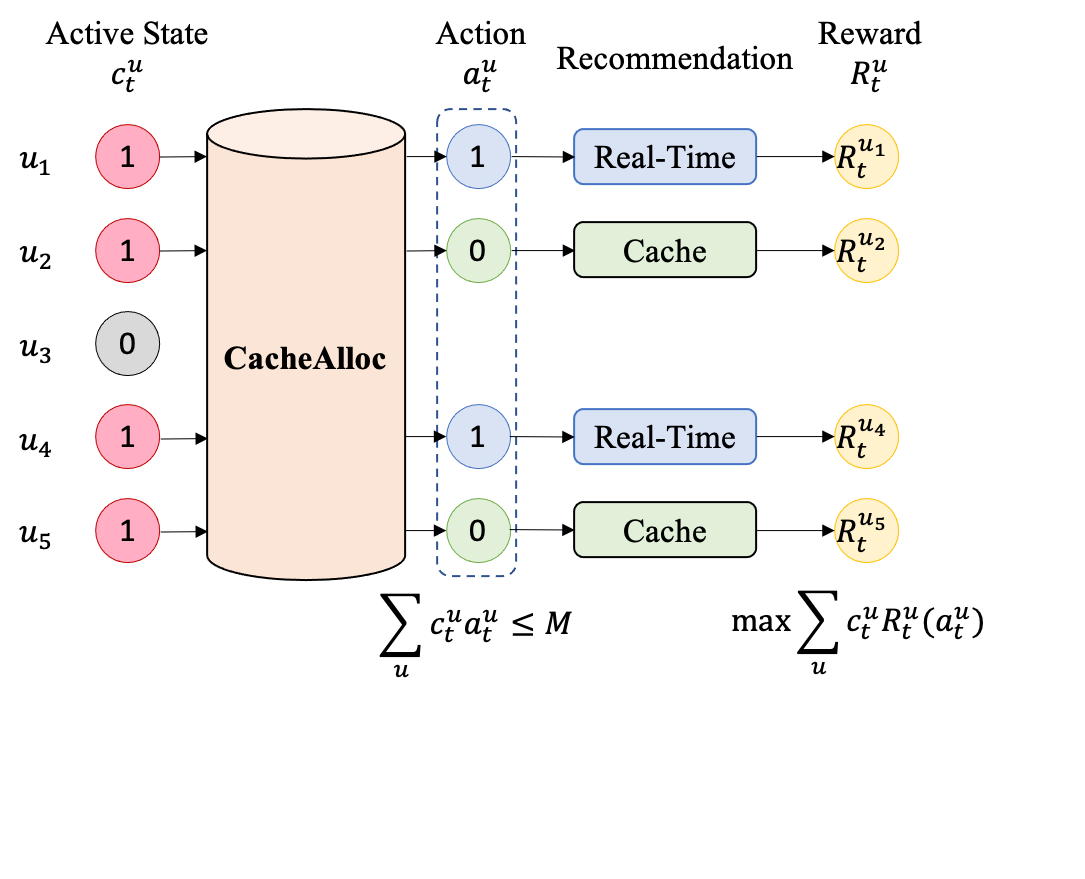}
    \caption{The Cache Allocation Problem.}
    \label{fig:cache_alloc}
    \vspace{-5mm}
\end{figure}

\begin{figure*}[t]
    \centering
    \includegraphics[width=0.8\linewidth,trim=0 40 0 0, clip]{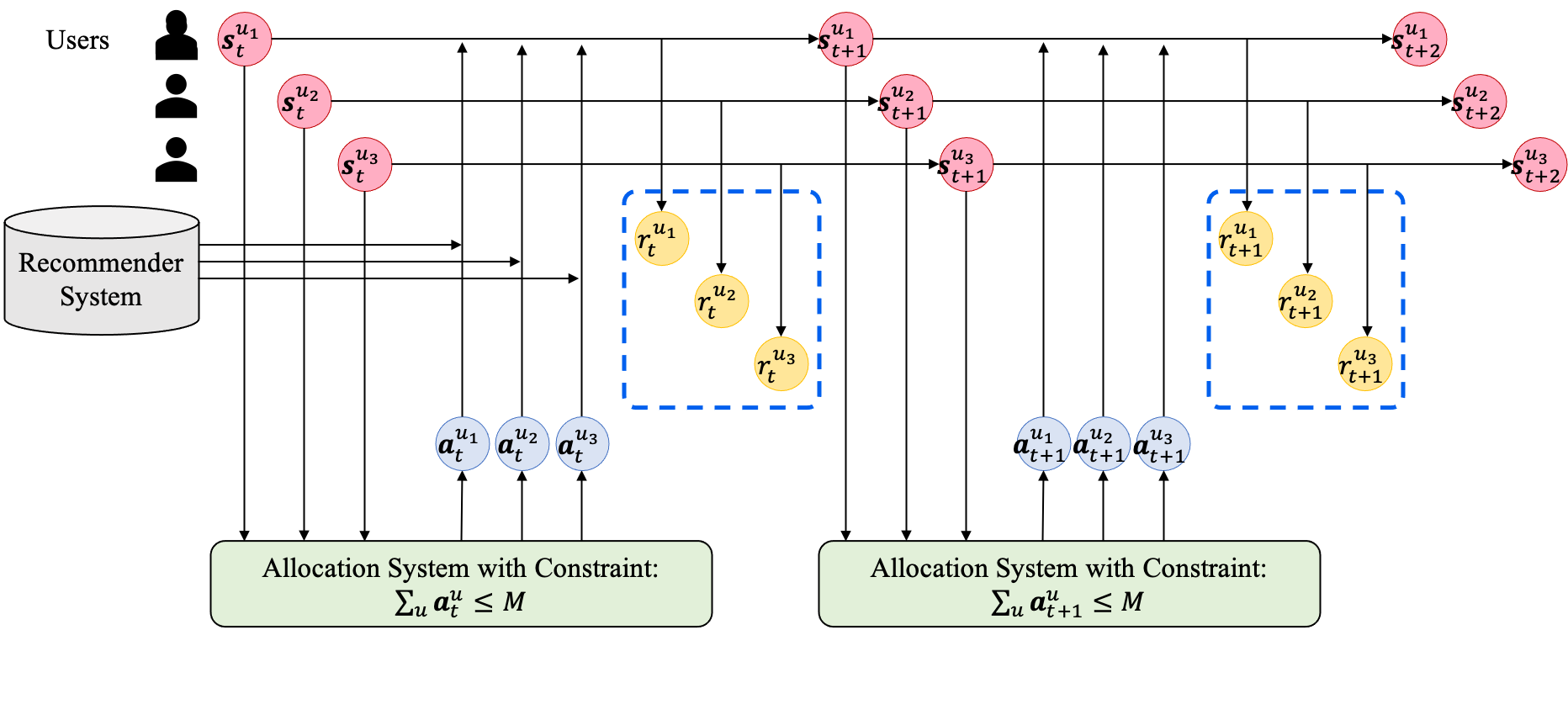}
    \caption{The Real Cache Allocation Problem.}
    \label{fig:cmdp}
    \vspace{-2mm}
\end{figure*}

\section{Cache Allocation Problem}
This section first provides a simplified cache allocation problem with pre-defined value estimation. Then, we discuss the challenges arising from the value-strategy dependency and provide a CMDP modeling of the real cache allocation problem, considering the dependency between the value estimation and the allocation strategy.

\subsection{Simplified Cache Allocation Problem}
We first consider a simplified cache allocation problem, of which the values of different cache choices are pre-defined. Such a cache allocation problem is a constrained optimization problem shown in Figure \ref{fig:cache_alloc}. We denote $\mathcal{U}$ as the set of users. At each time period $t$, some users in $\mathcal{U}$ open the app and send requests to the recommender system. We use $c_t^u = 1$ to denote that User $u$ sends a request at Time $t$, and $c_t^u=0$ to denote that User $u$ is inactive at Time $t$. When $c_t^u=1$, the system needs to recommend $K$ items to User $u$. There are two choices for the system to provide such a recommendation, represented by $a_{t}^u\in\{0,1\}$:
\begin{itemize}
    \item \textbf{Real-Time recommendation}. When $a_{t}^u=1$, the system will recommend by real-time computation. Specifically, there will be multiple stages, e.g., retrieval and ranking. In each stage, deep models will be used to predict the user's possible feedback (watch time, follow, like, etc.), and return $L$ items, of which the top $K$ items are recommended to the user and the rest are put into a result cache for future use.
    \item \textbf{Recommendation by the cache}. When $a_{t}^u=0$, the system will provide recommendations by the cache. In such cases, the recommender system obtains $K$ items directly from the user's result cache, and then updates the result cache.
\end{itemize}

Real-time recommendations usually perform better than cached recommendations since they can utilize users' most recent feedback by deep models \cite{liu2018stamp}. However, the computational burden of real-time recommendation is much larger, and the total requests to be processed by real-time recommendation must not exceed a given value $M$ in each time period. Under the above discussions, the cache allocation problem can be written as:

\noindent\textbf{CacheAlloc-Simplified}:
\begin{align}
    \max_{a_{t}^u} &\sum_{u\in\mathcal{U}} c_t^u\mathbb{E}\left[R_t^u |a_t^u\right]\label{eq:cache-alloc-obj} \\
\text{s.t.~}&\sum_{u\in\mathcal{U}}c_t^ua_{t}^u \leq M,\forall t \label{eq:cache-alloc-cost}\\
    &a_{t}^u \in\{0,1\},\forall u\in\mathcal{U},\forall t \label{eq:cache-alloc-x-t-j}
\end{align}
where $R_{t}^u$ means the value of User $u$ at Time $t$ when the system chooses the cache action $a_t^u$, and $M$ is the maximum requests that can be processed by real-time recommendation in each time period. The value $R_t^u$ can be any user feedback according to the concrete scenario, such as the usage time in news/short video recommendations and the revenue of advertising systems.

Suppose we have obtained the reward $R_t^u$ for each request, then the cache allocation problem in Eq. \eqref{eq:cache-alloc-obj}$\sim$\eqref{eq:cache-alloc-x-t-j} can be easily solved:
\begin{proposition} \label{prop:trivial}
    Given $\mathbb{E}\left[R_t^u|a_t^u\right]$ for each $u$ and each $a_t^u\in\{0,1\}$, the solution to the \textbf{CacheAlloc-Simplified} is:
    \begin{equation}
        a_t^u = \textbf{arg-top}_M\left\{\left.\mathbb{E}\left[R_t^u |a_t^u=1\right] - \mathbb{E}\left[R_t^u |a_t^u=0\right] \right|c_t^u=1\right\}
    \end{equation}
    where $\textbf{arg-top}_M$ means that $a_t^u=1$ if $u$ is in the top $M$ users with $c_t^u=1$ ranked by the given scores, and otherwise $a_t^u=0$.
\end{proposition}
\begin{proof}
See Appendix A.
\end{proof}

Proposition \ref{prop:trivial} means that the optimal allocations of the real-time recommendations are the top $M$ requests regarding the value difference between real-time and cached recommendations. Such results are also discussed in \cite{jiang2020dcaf}.

\begin{corollary} \label{corollary:equality}
The optimal solution to the \textbf{CacheAlloc-Simplified} satisfies $\sum_ua_t^u=M$, which means Eq. \eqref{eq:cache-alloc-cost} reaches the equality.
\end{corollary}

According to Corollary \ref{corollary:equality}, we will replace Eq. \eqref{eq:cache-alloc-cost} by the equality $\sum_ua_t^u=M$ in the rest of this paper.

\subsection{Real Cache Allocation Problem}
Although the solution to \textbf{CacheAlloc-Simplified} is simple according to Proposition \ref{prop:trivial}, it assumes that the value $R_t^u$ under each $a_t^u$ is pre-defined before we compute the cache allocation problem. Such an assumption is used in the existing computational resource allocation methods \cite{jiang2020dcaf,lu2023greenflow,zhou2023rl}, but it does not hold in the cache allocation problem. Moreover, \textbf{CacheAlloc-Simplified} assumes the cache choices can be determined after the information of all requests is obtained at each time step, which contradicts the streaming manner of online requests. Now, we discuss the two challenges and provide the modeling of the real cache-allocation problem.

\subsubsection{Value-Strategy Dependency}
The value-strategy dependency arises from the users' consecutive interactions with the recommender system against the finiteness of the cache. When a user opens the app, he/she interacts with the recommender system through multiple requests, and the system responds to each request. Based on the user's experience with the recommended items, the user decides whether to continue using the app or leave. Because the cache is finite, continuously providing cached recommendations to the same user's multiple requests will soon exhaust the result cache and deteriorate the user's experience, as shown in Figure \ref{fig:continuous-watch-time}. Therefore, the cache choice $a_t^u$ will affect the future value $R_{t+1}^u$, and the future cache choice $a_{t+1}^u$, which means the value $R_t^u$ should contain not only the user's current feedback but also the user's future feedback.

Now, we discuss the real cache allocation problem considering the value-strategy dependency. We use RL to model the real cache allocation problem, as shown in Figure \ref{fig:cmdp}. Specifically, we model the interaction between users and the recommender systems as a Constrained Markov
Decision Process (CMDP) \cite{sutton2018reinforcement} $\left(\mathcal{S},\mathcal{A},\mathcal{P},R,\mathcal{C},\rho_0,\gamma\right)$, where $\mathcal{S}$ is the state space, $\mathcal{A}$ is the action space (in this paper $\mathcal{A} = \{0,1\}$), $P:\mathcal{S}\times\mathcal{A}\to\mathcal{S}$ is the transition function, $R:\mathcal{S}\times\mathcal{A}\to\mathbb{R}$ is the reward function, $\mathcal{C}$ is the constraints, $\rho_0$ is the initial state, and $\gamma$ is the discounting factor.
When User $u$ opens the app, a session begins, which consists of multiple requests until the user leaves the app. At Step $t$, the recommender system obtains a user state $\boldsymbol{s}_t^u\in\mathcal{S}$.
According to the user state $\boldsymbol{s}_t^u$, the system generates an action $a_t^u\in\mathcal{A}$, and performs the real-time or cached recommendations according to the action $a_t^u$. After watching the recommended items, the user provides feedback $r_t^u = R\left(\boldsymbol{s}_t^u,a_t^u\right)$. Then, the user transfers to the next state $\boldsymbol{s}_{t+1}^u=P\left(\boldsymbol{s}_t^u,a_t^u\right)$ and determines whether to send the next request or leave. The constraint set $\mathcal{C}$ contains the computational budget constraints as shown in Eq. \eqref{eq:cache-alloc-cost}.

The CMDP model considers the impacts of the actions on future user states and rewards. Under the abovementioned discussions,  the rewards to maximize can be defined as:
\begin{equation}
\begin{aligned}
R_{t}^u&=\mathbb{E}\left[\left.\sum_{t'=t}^\infty \gamma^{t'-t}r_{t'}^u\right|\boldsymbol{s}_t^u, a_t^u\right]
\end{aligned}
\end{equation}

\subsubsection{Streaming Allocation} In the online scenario, the user requests arrive streamingly, and the system needs to make decisions as soon as possible. Therefore, it is impossible to obtain all the requests at the time period $t$ to calculate the optimization problem. In other words, the cache choice $a_t^u$ should be determined locally by the user state $\boldsymbol{s}_t^u$:
\begin{equation}
    a_t^u = \mu(\boldsymbol{s}_t^u;\theta)
\end{equation}
where $\theta$ is the parameter.

\subsubsection{Real Cache Allocation Problem}
Considering the value-strategy dependency and the streaming allocation problem, the real cache allocation problem can be written as:

\noindent\textbf{CacheAlloc-Real}:
\begin{equation} \label{eq:real-cache-alloc}
\begin{aligned}
    \max_{a_{t}^u} &\sum_{u\in\mathcal{U}} c_t^u\mathbb{E}\left[\left.\sum_{t'=t}^\infty \gamma^{t'-t}r_{t'}^u\right|\boldsymbol{s}_t^u, a_t^u\right] \\
\text{s.t.~}&\sum_{u\in\mathcal{U}}c_t^ua_{t}^u = M,\forall t \\
&a_t^u = \mu(\boldsymbol{s}_t^u;\theta) \\
    &a_{t}^u \in\{0,1\},\forall u\in\mathcal{U},\forall t 
\end{aligned}
\end{equation}
where we have replaced the budget constraint with an equality constraint according to Corollary \ref{corollary:equality}.

The \textbf{CacheAlloc-Real} considers the value-strategy dependency by the CMDP modeling. However, solving this problem is still challenging because of the budget constraint and the demand for streaming allocation. The next section proposes the RPAF approach to solve the problem.

\section{Methodology} \label{sec:method}
\begin{figure}[t]
    \centering
    \includegraphics[width=1.0\linewidth,trim=0 40 0 0, clip]{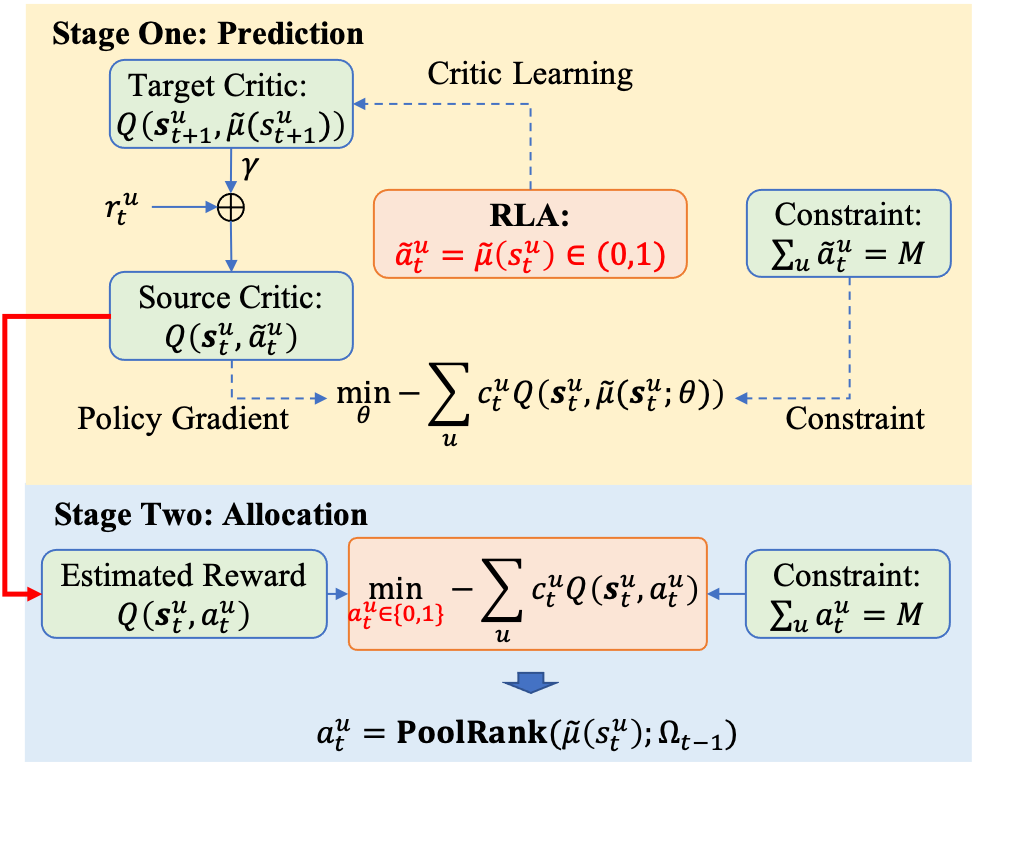}
    \caption{The RPAF method.}
    \label{fig:rrsa}
    \vspace{-5mm}
\end{figure}

This section proposes RPAF to solve the cache allocation problem while considering the impacts of the value-strategy dependency and streaming allocation. We first provide the overall framework in Sec. \ref{sec:method:framework}, which contains an RL-based prediction stage for value estimation under value-strategy dependency and an allocation stage to apply the streaming cache allocation. RPAF tackles the value-strategy dependency by the constrained RL modeling, and Sec. \ref{sec:method:rla} proposes a relaxed local allocator (RLA) to train the constrained RL. Then, Sec. \ref{sec:method:allocation} discusses the allocation stage and proposes a PoolRank algorithm to tackle the streaming allocation problem. 

\subsection{Overall Framework} \label{sec:method:framework}
RPAF, as shown in Figure \ref{fig:rrsa}, contains two stages, i.e., a prediction stage to estimate the value under different cache choices considering the value-strategy dependency and an allocation stage to solve the stream cache allocation problem under the estimated value.

The prediction stage uses the actor-critic structure. Specifically, we let the cache choice $a_t^u$ be determined by an actor function $\mu\left(\boldsymbol{s}_t;\theta\right)$ parameterized by $\theta$, and define the critic function $Q\left(\boldsymbol{s}_t^u,a_t^u;\phi\right)$, parameterized by $\phi$, to estimate the long-term reward $R_t^u$ in Eq. \eqref{eq:real-cache-alloc} under the action $a_t^u$. Then, the objective of critic learning is:
\begin{align}
    \min_{\phi} \sum_{u\in\mathcal{U}} c_t^u\left[Q\left(\boldsymbol{s}_t^u,a_t^u\right)-r_t^u-\gamma Q\left(\boldsymbol{s}_{t+1}^u,\mu(\boldsymbol{s}_{t+1}^u;\theta^-);\phi^-\right)\right]^2\label{eq:cac-critic}
\end{align}
And the objective of actor learning is:
\begin{align}
    \min_{\theta} &-\sum_{u\in\mathcal{U}} c_t^u Q\left(\boldsymbol{s}_t^u,\mu\left(\boldsymbol{s}_t^u;\theta\right)\right) \label{eq:cac-actor}\\
    \text{s.t. } &\sum_{u}c_t^u\mu\left(\boldsymbol{s}_{t}^u;\theta\right) = M \label{eq:cac-constraint}
\end{align}

After the prediction stage, we obtain the estimated critic function $Q\left(\boldsymbol{s}_t^u,a_t^u;\phi\right)$. Then, in the allocation stage, we solve the following constrained optimization problem:
\begin{equation} \label{eq:strict-allocator}
\begin{aligned}
    \max_{a_{t}^u} &\sum_{u\in\mathcal{U}} c_t^uQ\left(\boldsymbol{s}_t^u,a_t^u;\phi\right) \\
\text{s.t.~}&\sum_{u\in\mathcal{U}}c_t^ua_{t}^u = M,\forall t \\
    &a_{t}^u \in\{0,1\},\forall u\in\mathcal{U},\forall t 
\end{aligned}
\end{equation}

\bigskip

\noindent\textbf{Remark: Why not DQN?}
It seems DQN is more proper in the discrete decision problem, and there exists research on applying DQN to the computational resource allocation problem\cite{zhou2023rl}. The loss function of DQN is
\begin{equation}
    \mathcal{L}_{DQN} = c_t^u\left[Q\left(\boldsymbol{s}_t^u,a_t^u\right)-r_t^u-\gamma \max_{a\in\{0,1\}}Q\left(\boldsymbol{s}_{t+1}^u,a)\right)\right]^2\label{eq:dqn-critic}
\end{equation}
However, it is challenging to consider the \textbf{global} budget constraint, i.e. Eq. \eqref{eq:cache-alloc-cost} in the loss function. Specifically, because real-time recommendations are more likely to perform better than cached recommendations, we usually have $Q\left(\boldsymbol{s}_{t+1}^u,1\right) > Q\left(\boldsymbol{s}_{t+1}^u,0\right)$, and Eq. \eqref{eq:dqn-critic} will degrade to the following form:
\begin{equation}
    \mathcal{L}_{DQN} = c_t^u\left[Q\left(\boldsymbol{s}_t^u,a_t^u\right)-r_t^u-\gamma Q\left(\boldsymbol{s}_{t+1}^u,1)\right)\right]^2\label{eq:dqn-critic-degrade}
\end{equation}
which means we may expect $Q\left(\boldsymbol{s}_t^u,a_t^u\right)$ to be close to $r_t^u+\gamma Q\left(\boldsymbol{s}_{t+1}^u,1\right)$, of which the action in the time period $t+1$ is always 1. This result is clearly inconsistent with the budget constraint. Recently, there has been some research on constrained DQN and weakly coupled DQN\cite{kalweit2020deep,el2024weakly,zhou2023rl}, which uses the Lagrangian technique to modify the $Q$-function. However, we find that the Lagrangian multiplier is difficult to obtain in the online streaming settings (also see the experiments in Sec. \ref{sec:experiment}). Therefore, we do not use DQN in this paper.

\bigskip

The key technique in RPAF is to separate the value estimation and the cache allocation while considering their dependency. The next problem is to solve RPAF. If the budget constraint in Eq. \eqref{eq:cac-constraint} did not exist, the problem would degenerate into a typical actor-critic structure, and there is an amount of research on more effective algorithms, e.g. DDPG \cite{lillicrap2015continuous}, TD3 \cite{fujimoto2018addressing}, SAC \cite{haarnoja2018soft}, etc. However, the budget constraint brings significant challenges to solving the problem:
\begin{itemize}
    \item \textbf{Globality}. The budget constraint is over all the users in Time $t$, and the system needs to determine all actions according to all the user states, which is computationally unaffordable.
    \item \textbf{Strictness}. The budget constraint should be satisfied strictly. However, typical RL algorithms are usually solved in the sense of expectation, which is unsatisfactory.
\end{itemize}

We use the relaxed local allocator (RLA) and the PoolRank technique to address these two issues, respectively.

\subsection{Prediction Stage with RLA} \label{sec:method:rla}
The motivation of RLA is to transform the global and strict constraint in Eq. \eqref{eq:cac-constraint} into a local and relaxed constraint so that it can be solved by sample-wise gradient descent. 
We first relax the binary cache action $a_t^u$ to a continuous one $\tilde{a}_t^u\in[0,1]$. $\tilde{a}_t^u$ can be regarded as the probability of $a_t^u=1$.
We define RLA as $\tilde{a}_t^u = \tilde{\mu}\left(\boldsymbol{s}_t;\theta\right)$, i.e., the allocator to determine the relaxed cache action $\tilde{a}_t^u$. We reuse the critic $Q\left(\boldsymbol{s}_t^u,\tilde{a}_t^u;\phi\right)$ as the estimated long-term reward under the relaxed action $\tilde{a}_t^u$. By the probability explanation of $\tilde{a}_t^u$, we have
\begin{equation} \label{eq:critic-function-tilde}
    Q\left(\boldsymbol{s}_t^u,\tilde{a}_t^u;\phi\right) = \tilde{a}_t^uQ\left(\boldsymbol{s}_t^u,1;\phi\right) + \left(1-\tilde{a}_t^u\right)Q\left(\boldsymbol{s}_t^u,0;\phi\right)
\end{equation}
By replacing the true action with RLA, we can rewrite the actor learning in Eq. \eqref{eq:cac-actor}\eqref{eq:cac-constraint} as:
\begin{equation} \label{eq:cac-relax-constraint-m}
\begin{aligned}
    \min_{\theta} &-\sum_{u\in\mathcal{U}} c_t^u Q\left(\boldsymbol{s}_t^u,\tilde{\mu}\left(\boldsymbol{s}_t^u;\theta\right)\right) \\
    \text{s.t. } &\frac{\sum_{u}c_t^u\tilde{\mu}\left(\boldsymbol{s}_{t}^u;\theta\right)}{\sum_{u}c_t^u} = m_t 
\end{aligned}
\end{equation}
Here, we rewrite the constraint by dividing by $\sum_uc_t^u$ on both sides of the equality, and so $m_t = M/\sum_uc_t^u$. $m_t \in (0,1)$ is the ratio of real-time recommendations at Time $t$.

With the proposed RLA, i.e. $\tilde{\mu}\left(\boldsymbol{s}_t^u;\theta\right)$, we are ready to derive the training process. We consider a penalty function $T\left(\hat{x},x\right)$ satisfying:
\begin{itemize}
    \item $T\left(\hat{x},x\right)$ reaches the minimum at $\hat{x}=x$.
    \item $T$ is convex with regard to $\hat{x}$.
\end{itemize}
Typically, $T(\cdot)$ can be chosen as:
\begin{itemize}
    \item mean square error (MSE) function: $T(\hat{x},x)=\left(
    \hat{x}-x\right)^2$.
    \item Kullback–Leibler (KL) divergence:
    
    $T(\hat{x},x)=-\left[x\log\left(\hat{x}\right)+(1-x)\log\left(1-\hat{x}\right)\right]$
\end{itemize}

By using the penalty function to replace the constraint, we combine the constraint into the objective of the actor learning:
\begin{align}
    \min_{\theta} L_a(\theta)=&\sum_{u\in\mathcal{U}} c_t^u \left[-Q\left(\boldsymbol{s}_t^u,\tilde{\mu}\left(\boldsymbol{s}_t^u;\theta\right)\right) + \alpha T\left(\frac{\sum_{u}c_t^u\tilde{\mu}\left(\boldsymbol{s}_{t}^u;\theta\right)}{\sum_{u}c_t^u}, m_t \right)\right]\label{eq:cac-relax-actor}
\end{align}

By minimizing $L_a(\theta)$, we obtain the policy $\tilde{\mu}$, which maximizes the critic function $Q$ under the budget constraints. However, the penalty function $T$ in Eq. \eqref{eq:cac-relax-actor} still contains a global sum over all users, which is computationally unaffordable. To tackle this challenge, we consider the upper bound of $L_a(\theta)$. According to the convexity of $T(\hat{x},x)$ regarding $\hat{x}$, we have
\begin{equation}
T\left(\frac{\sum_{u}c_t^u\tilde{\mu}\left(\boldsymbol{s}_{t}^u;\theta\right)}{\sum_{u}c_t^u}, m_t \right) \leq \frac{\sum_{u}c_t^uT(\tilde{\mu}\left(\boldsymbol{s}_t^u;\theta\right),m_t)}{\sum_uc_t^u}
\end{equation}
Then we obtain an upper bound $\bar{L}_a(\theta)$ satisfying $L_a(\theta)\leq \bar{L}_a(\theta)$:
\begin{equation}
    \bar{L}_a(\theta)=\sum_{u\in\mathcal{U}} c_t^u\left[ -Q\left(\boldsymbol{s}_t^u,\tilde{\mu}\left(\boldsymbol{s}_t^u;\theta\right)\right) + \alpha T\left(\tilde{\mu}\left(\boldsymbol{s}_{t}^u;\theta\right), m_t \right)\right]\label{eq:cac-relax-actor-upper}
\end{equation}

By minimizing the upper bound $\bar{L}_a\left(\theta\right)$, the original loss $L_a\left(\theta\right)$ will also be bounded from above. Note that terms of different $u$ in Eq. \eqref{eq:cac-relax-actor-upper} are independent. In the training process, for each sample containing a sample of a user $u$ at the time period $t$ satisfying $c_t^u=1$, the actor loss can be written in the local form:
\begin{equation}
    \bar{L}_a(\theta;u,t)= -Q\left(\boldsymbol{s}_t^u,\tilde{\mu}\left(\boldsymbol{s}_t^u;\theta\right)\right) + \alpha T\left(\tilde{\mu}\left(\boldsymbol{s}_{t}^u;\theta\right), m_t \right)\label{eq:cac-actor-sgd}
\end{equation}
and the critic loss of each sample is:
\begin{equation}
    L_c\left(\phi;u,t\right)=\left[Q\left(\boldsymbol{s}_t^u,a_t^u\right)-r_t^u-\gamma Q\left(\boldsymbol{s}_{t+1}^u,\tilde{\mu}(\boldsymbol{s}_{t+1}^u;\theta^-);\phi^-\right)\right]^2 \label{eq:cac-critic-sgd}
\end{equation}

We provide the training algorithm with RLA in Algorithm \ref{alg:prediction}.

\begin{algorithm}
\caption{Training with RLA in the prediction stage}
\label{alg:prediction}
\begin{algorithmic}[1]
\STATE Input: $\left\{c_{1:T}^u\boldsymbol{s}_{1:T}^u,a_{1:T}^u, r_{1:T}^u\right\}$ for each user $u \in \mathcal{U}$.
\STATE Output: A critic function $Q\left(\boldsymbol{s}_t^u,\tilde{a}_t^u;\phi\right)$ parameterized by $\phi$; the RLA: $\tilde{\mu}\left(\boldsymbol{s}_t;\theta\right)$ parameterized by $\theta$.
\FOR{each time period $t, 1\leq t\leq T$ from the replay buffer}
    \FOR{each user $u \in \mathcal{U}$ where $c_t^u = 1$}
        \STATE Collect the user state $\boldsymbol{s}_t^u$, the reward $r_t^u$, and the action $a_t^u$ from the replay buffer.
        \STATE Critic Learning: $\phi\gets \phi - \alpha_c\partial L_c/\partial \phi$, where $L_c$ is defined in Eq. \eqref{eq:cac-critic-sgd}.
        \STATE Actor Learning: $\theta\gets \theta - \alpha_a\partial \bar{L}_a/\partial \theta$, where $\bar{L}_a$ is defined in Eq. \eqref{eq:cac-actor-sgd}.
    \ENDFOR
\ENDFOR
\end{algorithmic}
\end{algorithm}

\subsection{Streaming Allocation with PoolRank} \label{sec:method:allocation}
After the prediction stage, we have obtained the critic $Q(\boldsymbol{s}_t^u,\tilde{a}_t^u;\phi)$. Because the strict action $a_t^u$ is a special case of the relaxed action $\tilde{a}_t^u$, $Q(\boldsymbol{s}_t^u,\tilde{a}_t^u;\phi)$ can also be used in the allocation stage in Eq. \eqref{eq:strict-allocator}. Now we discuss the solution to the streaming allocation problem. Existing approaches \cite{jiang2020dcaf,zhou2023rl} usually use the Lagrangian technique to make the computational resource allocation algorithm satisfy the streaming fashion, but the convergence of the Lagrangian multiplier is also a significant challenge. Luckily, the structure of the cache allocation problem makes it possible to find more effective approaches. Here, we propose the PoolRank algorithm to perform the streaming allocation in the allocation stage of RPAF.

We first show the relationship between the final allocation strategy $a_t^u$ and the RLA output $\tilde{\mu}\left(\boldsymbol{s}_t^u;\theta\right)$:
\begin{proposition} \label{prop:consistency}
Assume the critic $Q(\boldsymbol{s}_t^u,\tilde{a}_t^u;\phi)$ and the RLA $\tilde{\mu}\left(\boldsymbol{s}_t^u;\theta\right)$ are obtained from the prediction stage, i.e. they satisfies Eq. \eqref{eq:cac-relax-actor-upper}. Then, the solution to the allocation stage in Eq. \eqref{eq:strict-allocator} is
\begin{equation} \label{eq:strict-allocator-actor}
a_t^u = \textbf{arg-top}_M\left\{\left.\tilde{\mu}\left(\boldsymbol{s}_t^u;\theta\right)\right|u\in\mathcal{U},c_t^u=1\right\}
\end{equation}
\end{proposition}
\begin{proof}
    See Appendix B.
\end{proof}

Proposition \ref{prop:consistency} shows the relationship between the relaxed allocator $\tilde{\mu}$ and the final allocation result $a_t^u$. Actually, if a request is more desired for the real-time recommendation, RLA will also return a larger value. To illustrate the PoolRank algorithm, we rewrite Eq. \eqref{eq:strict-allocator-actor} as
\begin{equation} \label{eq:online-ranking-ideal}
\begin{aligned}
a_t^u &= \textbf{1}\left\{ \textbf{Rank}\left(\tilde{\mu}\left(\boldsymbol{s}_{t}^{u}\right)| \Omega_{t}\right) \leq M\right\}\\
\text{s.t. }\Omega_{t} &= \{\tilde{\mu}(\boldsymbol{s}_{t}^{u};\theta)|u\in\mathcal{U},c_t^u=1\}
\end{aligned}
\end{equation}
where $\Omega_t$ is the set of all LRA outputs $\tilde{\mu}\left(\boldsymbol{s}_t^u\right)$ at the time period $t$, and the \textbf{Rank} operator computes the rank of $\tilde{\mu}\left(\boldsymbol{s}_t^u\right)$ in $\Omega_t$ ordered from the largest to the smallest. Now, we modify Eq. \eqref{eq:online-ranking-ideal} to a streaming fashion. Specifically, we introduce a detailed time $\tau$, where the time period $t$ is regarded as $\{\tau:t\leq \tau < t+1\}$. Assuming the requests come at $\tau$, the rank function in Eq. \eqref{eq:online-ranking-ideal} should be modified to
\begin{equation} \label{eq:online-ranking-wrong}
\begin{aligned}
a_\tau^u &= \textbf{1}\left\{ \textbf{Rank}\left(\tilde{\mu}\left(\boldsymbol{s}_{\tau}^{u}\right)| \Omega_{t}\right) \leq M\right\}\\
\text{s.t. }
\Omega_{t} &= \{\tilde{\mu}(\boldsymbol{s}_{\tau}^{u};\theta)|u\in\mathcal{U},c_\tau^u=1,t\leq \tau < t+1\}
\end{aligned}
\end{equation}
where the time of each request is replaced by the detailed time $\tau\in[t,t+1)$, while the $\Omega_t$ is still the set of the RLA values of all the requests arriving from $t$ to $t+1$.

Here, the challenge of streaming allocation is evident: we cannot obtain $a_\tau^u$ because it needs to obtain $\Omega_t$, which contains the requests coming after $\tau$. To tackle this challenge, we can approximate $\Omega_t$ by $\Omega_{t-1}$, because the real-time request volume at the time period $t$ will not differ much from the time period $t-1$.

We call $\Omega_t$ the request pool, which contains a set of requests with the LRA values. Then, the PoolRank formulation is

\noindent\textbf{PoolRank}:
\begin{equation} \label{eq:online-ranking-poolrank}
\begin{aligned}
a_\tau^u &= \textbf{1}\left\{ \textbf{Rank}\left(\tilde{\mu}\left(\boldsymbol{s}_{\tau}^{u}\right)| \Omega_{t-1}\right) \leq M\right\}\\
\text{s.t. }
\Omega_{t-1} &= \{\tilde{\mu}(\boldsymbol{s}_{\tau}^{u};\theta)|u\in\mathcal{U},c_\tau^u=1,t-1\leq \tau < t\}
\end{aligned}
\end{equation}

The online PoolRank algorithm is shown in Algorithm \ref{alg:poolrank}. Besides the main formulation in Eq. \eqref{eq:online-ranking-poolrank}, we also apply the following techniques to accelerate the computation:
\begin{itemize}
    \item \textbf{Bucketization}: the rank operator in Eq. \eqref{eq:online-ranking-poolrank} can be computed in a $O(1)$ time complexity by bucketization. We first bucketize $\tilde{\mu}\left(\boldsymbol{s}_\tau^u\right)$ to $I_\tau^u$. Then, a prefix array $A$ stores the number of requests in $\Omega_t$ with values larger than $I_\tau^u$. Then, the rank of the request can be obtained by a simple searching process, i.e., Line 6 in Algorithm \ref{alg:poolrank}.
    \item \textbf{Dual-Buffer Mechanism}: note that the prefix array $A$ can be computed asynchronously. We use $A_{online}$ in the online process and update $A_{online}$ once a new $A$ is available. This mechanism avoids updating the prefix array $A$ online, which significantly accelerate the PoolRank process.
\end{itemize}

As a summary of this section, we provide RPAF to solve the cache allocation problem. The prediction stage utilizes RLA to learn the critic and policy function, as shown in Algorithm \ref{alg:prediction}, while the allocation stage utilizes PoolRank to perform a streaming allocation, as shown in Algorithm \ref{alg:poolrank}. The next section will validate the effectiveness of RPAF by offline and online experiments.

\begin{algorithm}
\caption{PoolRank}
\label{alg:poolrank}
\begin{algorithmic}[1]
\STATE \textbf{Online process} (From the perspective of request level):
\STATE Input: $\left\{\tilde{\mu}\left(\boldsymbol{s}_{\tau};\theta\right)\right\}$ for each user $u \in \mathcal{U}$, at Time ${\tau}$.
\STATE Output: The final decision ${a}_{\tau}^u: a_{\tau}^u \in\{0,1\},\forall u\in\mathcal{U},\forall t $
\FOR{$\tau | t \leq \tau \leq t+1$ and $c_{\tau}^u = 1$}
    \STATE ${I}_{\tau}^u = \lfloor\tilde{\mu}\left(\boldsymbol{s}_{\tau};\theta\right) / \eta\rfloor$, where $\eta$ is the resolution.
    \STATE ${o}_\tau^u = {A}_{online}[I_\tau^u]$, where ${o}_\tau^u$ is ${Rank}(\tilde{\mu}\left(\boldsymbol{s}_{t}^{u}\right))$.
    \STATE ${C}_{t}[I_\tau^u]={C}_{t}[I_\tau^u]+1$, where $C$ is a counting array.
    \STATE $a_{\tau}^u = if\left({o}_{\tau}^u < M;1;0\right)$
\ENDFOR
\STATE
\STATE \textbf{Asynchronous process} (From a global system perspective):
\FOR{i in range (0, $\lfloor1/{\eta}\rfloor$)}
    \STATE ${A}_t = {A}_{t-1}[i] + C_{t}[i]$, where $A$ is a prefix array.
\ENDFOR
\STATE $A_{online} = A_{t}$
\end{algorithmic}
\end{algorithm}

\section{Experiments} \label{sec:experiment}
Extensive offline and online experiments are conducted to address the following research questions (RQs):
\begin{itemize}
\item \textbf{RQ1}: How does RPAF perform compared to other state-of-the-art methods?
\item \textbf{RQ2}: Does RLA effectively estimate the value of cache choices while considering the impact of allocation?
\item \textbf{RQ3}: Can the PoolRank algorithm better allocate real-time and cached recommendations streamingly?
\item \textbf{RQ4}: Can RPAF improve the performance of real-world recommender systems with a significant amount of cached recommendations?
\end{itemize}
\subsection{Offline Experiment Settings}
\subsubsection{Dataset and Metric}
We select the KuaiRand \cite{gao2022kuairec} dataset as the offline experimental dataset. KuaiRand is a public dataset from Kuaishou, comprising 27,285 users and 32,038,725 items. It provides contextual features of users and items and various user feedback signals such as watch time, likes, follows, and more. We regard watch time as the user feedback signal, i.e., the reward $r_t^u$ represents the watch time of User $u$ at Time $t$. The performance metric is the average \textbf{WatchTime}, i.e., the accumulated watching time per user.

\subsubsection{Simulator of Feedback}
The KuaiRand dataset provides user feedback on given items, but it does not consider the impact of the cache. To simulate the user's behavior under different cache choices, we construct a simulator to mimic the interaction between user requests and the recommender system. Upon receiving a user request, the recommender system determines the cache choice $a_t^u$. If $a_t^u=1$, the system selects the real-time recommendation, including matching and ranking. The real-time recommendation returns $L = 40$ items for each request, with the top $K = 8$ items being recommended to the user while the remaining items are put into the user's result cache. When $a_t^u=0$, the system chooses the cached recommendation, which returns the top $K$ items from the result cache and subsequently updates the result cache by removing these $K$ items. If the result cache is empty, the system must select real-time recommendations. If the computational resource of the current time period has been exhausted, the system must choose cached recommendations or fail to recommend. When receiving real-time recommendations, the simulator employs a pre-trained model to simulate user feedback. When receiving cached recommendations, the simulator modifies the original predicted feedback by a discounting factor shown in Figure \ref{fig:continuous-watch-time}. The simulator is used to evaluate the performance of different cache allocation methods.

\subsubsection{Details}
The KuaiRand dataset is an hourly dataset, with the number of requests per hour ranging from 1000 to 8000. We set the maximum real-time recommendations per hour as 4500. Hence, the peak ratio of cached recommendations is 43.7\%.
In the RL problem during the prediction stage, the state $\boldsymbol{s}_t$ contains the user profiles, browsing records, request context, and the current remaining content in the results cache. The action value $\tilde{a}_{t}^u$ indicates the probability that this request should opt for real-time recommendation instead of cached recommendation at the current time step. The reward $r_t^u$ is the watch time of the user at the recommended items.
we keep the same model backbone (i.e. the model architecture for the actors), which is a five-layer multi-layer perception, for all the baselines and the proposed method for a fair comparison.
For each experiment, 20 trials are conducted to calculate the mean and standard deviation. The other hyper-parameters are presented in Table 3 in Appendix C. 

\begin{table}[t]
\caption{Performance of offline experiments.}
\vspace{-3mm}
\centering
\begin{tabular}{c|c}
    \hline
    Methods & WatchTime (s)  \\
    \hline
    Greedy & 1750 \\
    \hline
    CRAS & 1821($\pm$21.8) \\
    \hline
    DCAF & 1827($\pm$41.2) \\
    \hline
    RL-MPCA & 1899($\pm$34.9) \\
    \hline
    RPAF-DDPG-w/o RLA & 1911($\pm$21.3) \\
    RPAF-TD3-w/o RLA & 1969($\pm$22.7) \\
    \hline
    RPAF-DDPG-KL & 2057($\pm$27.8) \\
    RPAF-DDPG-MSE & 2049($\pm$31.2) \\
    \hline
    RPAF-TD3-KL & 2128($\pm$19.3) \\
    \textbf{RPAF-TD3-MSE} & \textbf{2146($\pm$24.2)} \\
    \hline
    \textit{All Real-Time (Ideal)} & \textit{2347} \\
    \hline
\end{tabular}
\vspace{-2mm}
\label{table:offline-exp}
\end{table}

\begin{figure}[t]
    \centering
    \includegraphics[width=0.95\linewidth,trim=0 12 0 0, clip]{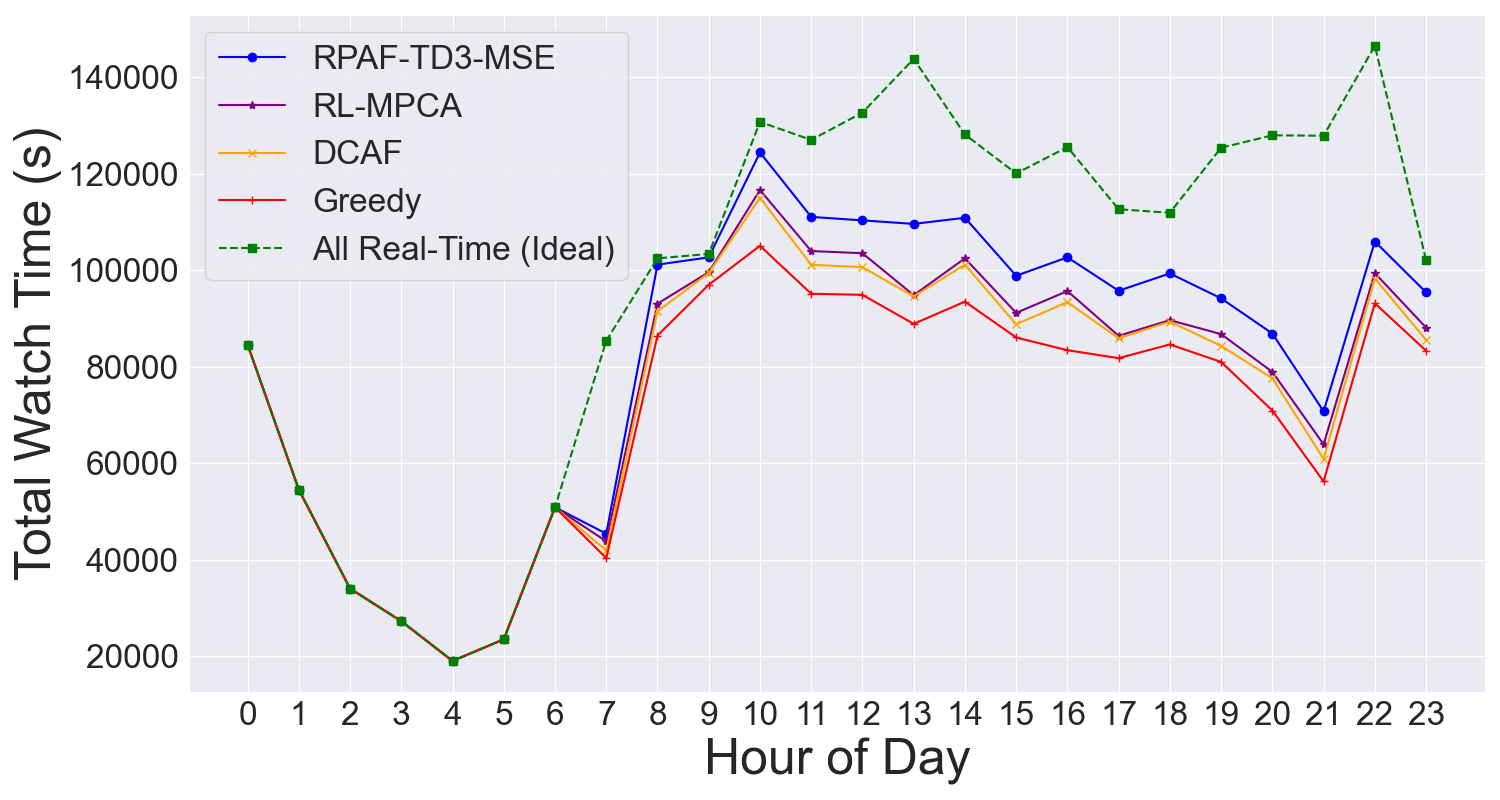}
    \caption{Performance comparison over one day.}
    \label{fig:base_reward_compare}
\vspace{-5mm}
\end{figure}

\subsubsection{Compared Methods} The following methods are compared:
\begin{itemize}
\item \textbf{Greedy}: The greedy approach prioritizes real-time recommendations over cached recommendations as long as the current system traffic load allows. Since requests come streaming, this method can always meet budget constraints.
\item \textbf{DCAF}: DCAF\cite{jiang2020dcaf} formulates the computational resource allocation problem as a constrained optimization problem and subsequently solves it with linear programming algorithms.
\item \textbf{CRAS}: CRAS\cite{yang2021computation} also provides a constrained optimization model and solves it by a multi-step feedback control method.
\item \textbf{RL-MPCA} \cite{zhou2023rl}: RL-MPCA formulates the computational resource allocation problem as a Weakly Coupled MDP problem and solves it with DQN.
\item \textbf{RPAF}: The proposed RPAF method. Several details are to be determined in the ablation study of RPAF.
\begin{itemize}
    \item The RL backbone of the actor-critic structure to learn Eq. \eqref{eq:cac-actor-sgd}\eqref{eq:cac-critic-sgd}: we choose two backbones, i.e. \textbf{DDPG} \cite{lillicrap2015continuous} and \textbf{TD3} \cite{fujimoto2018addressing}. Note that the choice of the RL backbones is independent of our work, and that we have merely chosen two representative methods as the backbone.
    \item The penalty function $T\left(\hat{x},x\right)$ in RLA: we choose three kinds of penalty functions, namely the KL divergence (denoted by \textbf{KL}), the MSE penalty, and no penalty function (denoted by \textbf{w/o RLA}).
\end{itemize}

\end{itemize}
    In the rest of this section, we use \textbf{RPAF-TD3-MSE} to denote the RPAF method with the TD3 backbone and the MSE penalty function, and other denotations are similar.
\begin{figure}[t]
    \centering
    \includegraphics[width=0.95\linewidth,trim=0 12 0 0, clip]{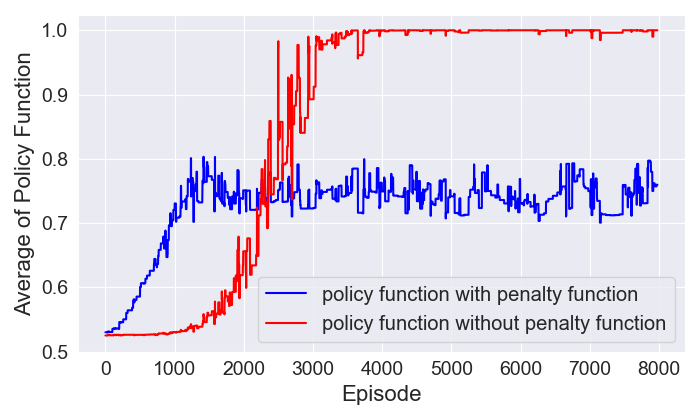}
    \caption{Training with and without penalty functions.}
    \label{fig:constraint_compare_compare}
    \vspace{-2mm}
\end{figure}
\begin{figure}[t]
    \centering
    \includegraphics[width=0.95\linewidth,trim=0 12 0 0, clip]{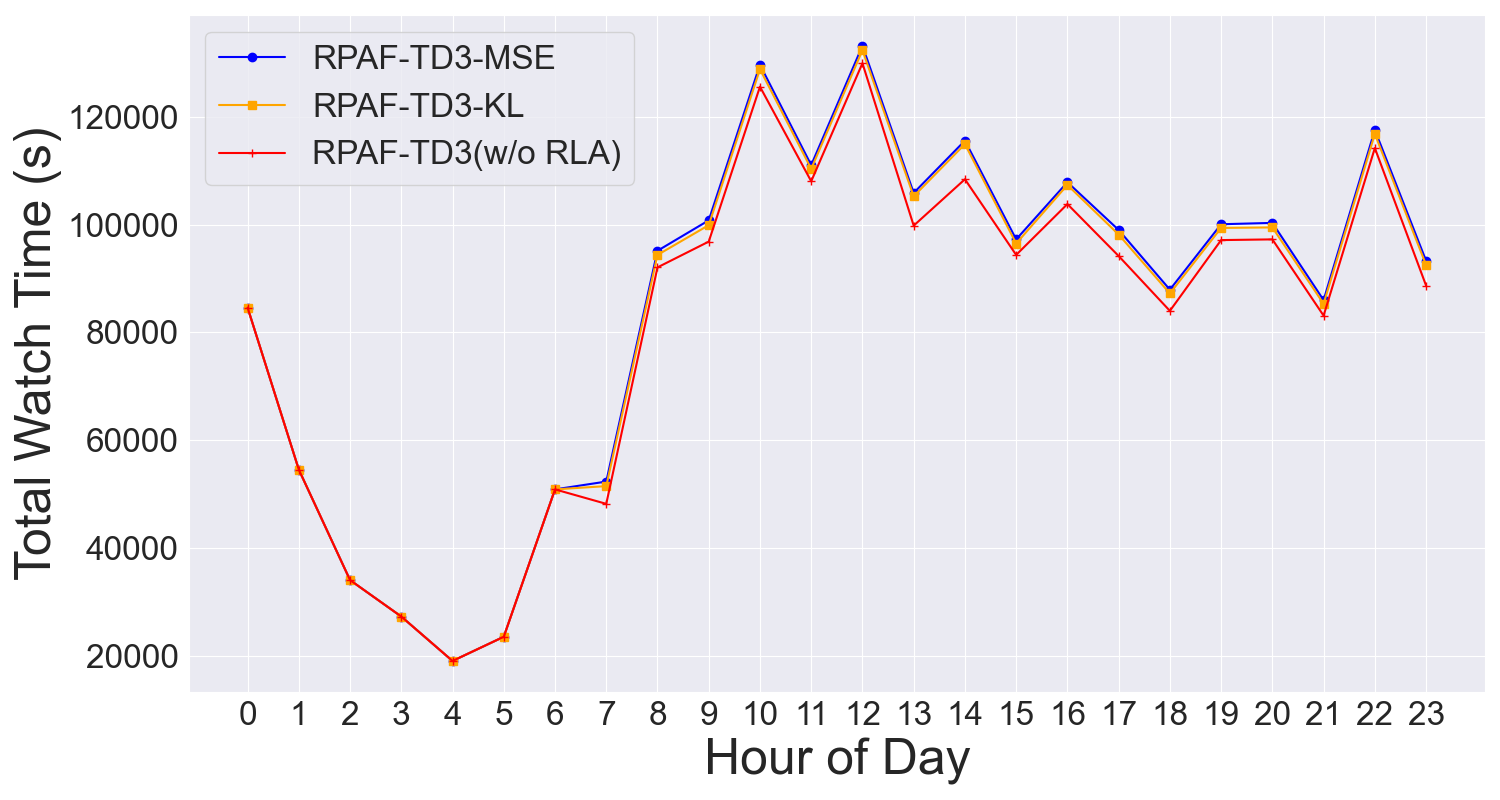}
    \caption{Performance comparison of penalty functions.}
    \label{fig:constraint_loss_compare}
    \vspace{-2mm}
\end{figure}
\subsection{Performance Comparison (RQ1)} 
We compare several methods as demonstrated in Table \ref{table:offline-exp}. To get the upper bound of the performance, an ideal strategy, \textbf{All Real-Time}, is included. This strategy chooses real-time recommendations regardless of the computational burden. As shown in Table \ref{table:offline-exp}, all dynamically allocated methods outperform the greedy approach. RPAF outperforms existing approaches, demonstrating the improvement achieved by considering the value-strategy dependency in cache allocation problems. Moreover, the comparison between RPAF with and without RLA demonstrates that RLA provides better performance by estimating the value of cache choice considering the constraint. A detailed discussion of RLA can be found in Section \ref{sec:experiment:rla}. Finally, RPAF with TD3 outperforms RPAF with DDPG, indicating that the RL backbone also impacts performance. Given that RPAF is independent of the RL backbone, any new studies on the actor-critic structure can be integrated into RPAF.

Additionally, Figure \ref{fig:base_reward_compare} illustrates the performance comparison across different time periods within a single day. During off-peak periods, for instance, from 0 AM to 6 AM, the total number of requests does not exceed the computational budget. Consequently, all the methods return real-time recommendations, resulting in the same reward. During peak periods, e.g., from 7 PM to 10 PM, cached recommendations become essential, and the performance of each method varies, of which RPAF achieves the best results.

\subsection{Impacts of RLA (RQ2)} \label{sec:experiment:rla}

This subsection validates the impacts of RLA. We regard the RPAF without RLA as RPAF with no penalty function $T(\hat{x},x)$. Figure \ref{fig:constraint_compare_compare} shows the convergence curve of the average output of the allocator $\tilde{\mu}\left(\boldsymbol{s}_t^u;\theta\right)$ with and without the penalty function $T(\hat{x},x)$. The policy function $\tilde{\mu}$ is shown to collapse to $1$ without the penalty function $T$. This means that all requests will be processed through real-time recommendation, which violates the budget constraint. In contrast, RLA with localized penalty functions is observed to converge to a reasonable ratio of real-time recommendations. As shown in Table \ref{table:offline-exp} and Figure \ref{fig:constraint_loss_compare}, RPAF with RLA significantly outperforms RPAF without RLA, while the two penalty functions chosen(KL and MSE) yield similar performance, indicating that RPAF is not largely dependent of the choice of penalty functions in RLA.

\subsection{Impacts of PoolRank (RQ3)}
Figure \ref{fig:computational_cost_compare} shows the number of real-time recommendations per hour under different methods. It is assumed that requests come in a streaming manner, and the system needs to determine the cache choice immediately upon receipt of a request, while ensuring that the total real-time recommendation per hour does not exceed the budget. If the system chooses a real-time recommendation while the system has no buffer, the request will still be downgraded into the cached recommendation. We also consider the RPAF without PoolRank, which determines cache choices solely in accordance with the allocator $\tilde{\mu}$. It is shown that DCAF, RL-MPCA, and RPAF without PoolRank cannot fully utilize all available real-time traffic, while the introduction of PoolRank enables the full exploitation of computational resources. This phenomenon shows the difference between batch allocation and streaming allocation. In streaming allocation, it is impossible to determine the cache choice according to all the requests in a time period. Existing methods may underestimate the available real-time computation resources in the future. PoolRank, however, achieves better performance due to its utilization of the most recent information on computational resources.

\begin{figure}[t]
    \centering
    \includegraphics[width=1.0\linewidth,trim=0 10 0 0, clip]{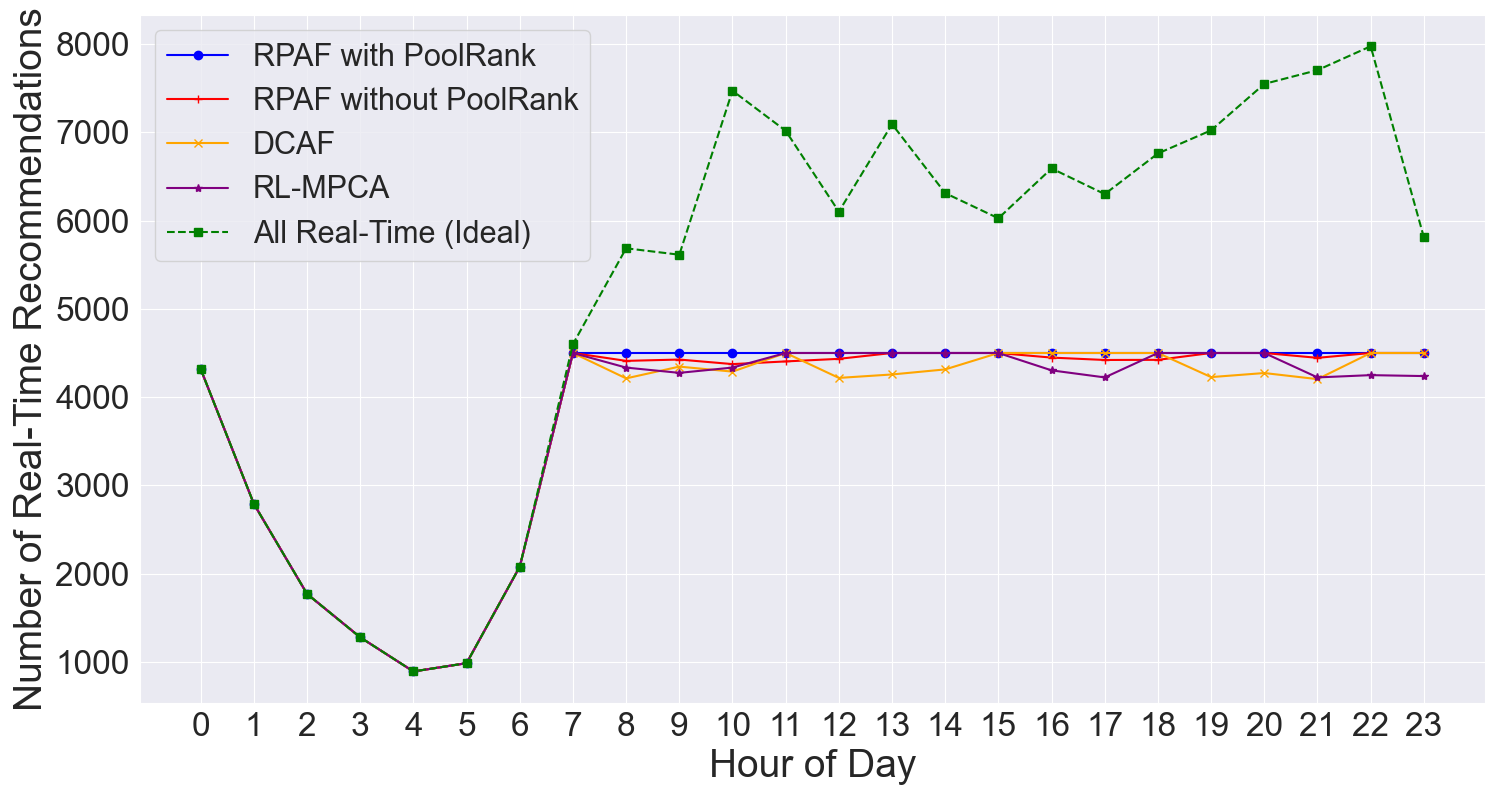}
    \caption{Computational cost comparison.}
    \label{fig:computational_cost_compare}
    \vspace{-3mm}
\end{figure}

\subsection{Online Experiments (RQ4)}
\begin{figure}[t]
    \centering
    \includegraphics[width=0.8\linewidth,trim=0 5 0 0, clip]{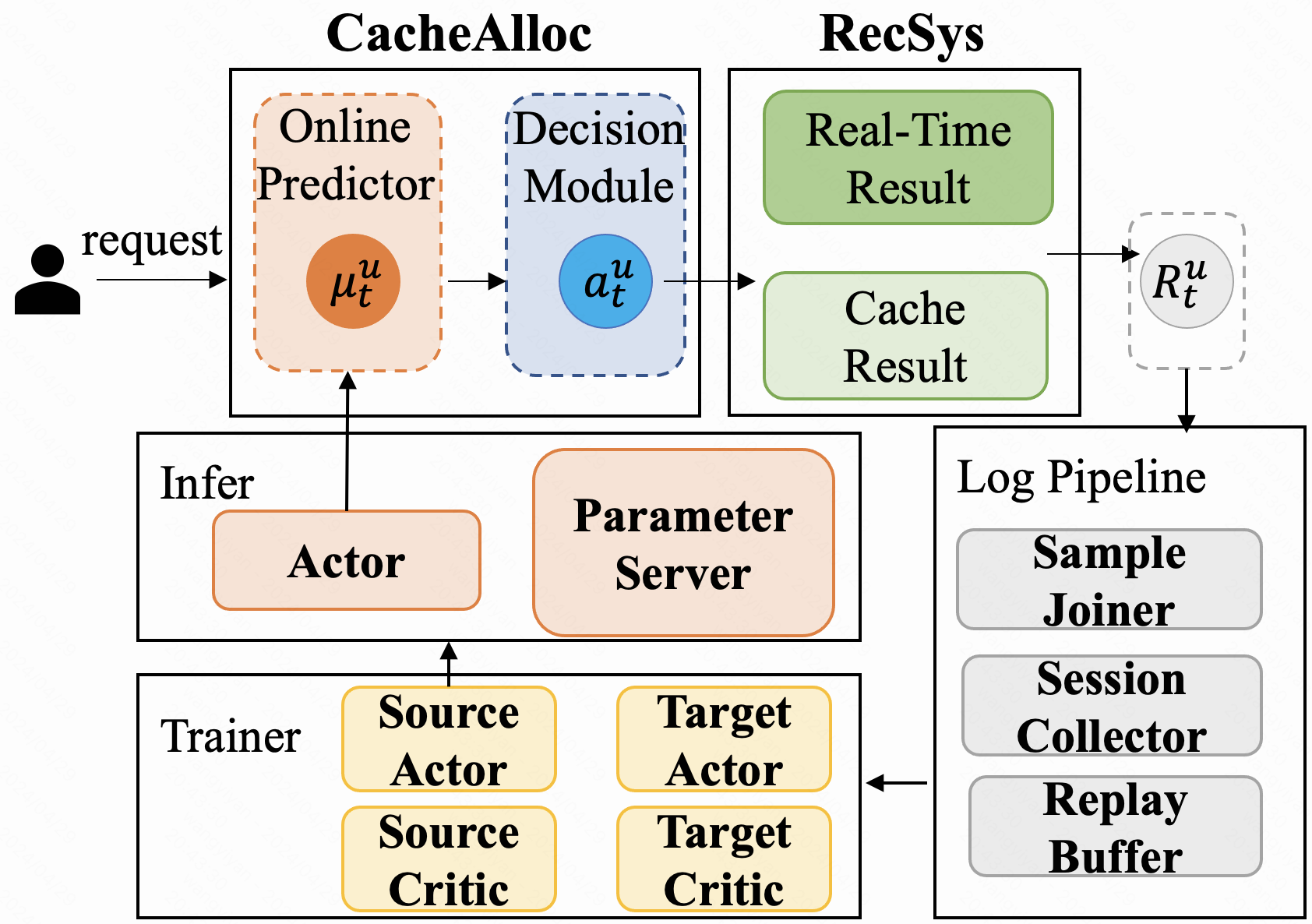}
    \caption{The Cache Allocation Online Process.}
    \label{fig:cache_alloc_online}
\end{figure}

We test the proposed RPAF method in Kwai, an online short video application with over 100 million users. The RPAF structure in the online system is shown in Figure \ref{fig:cache_alloc_online}. The model is trained continuously online. When a user's session ends, the session's data is immediately fed into the model for training, and the updated model is instantly deployed to the online service. We deploy the RPAF with the TD3 backbone and the MSE penalty function. There are 40\% cached recommendations in the peak period. In this scenario,  users' experience is highly consistent with their daily usage time. Thus, we regard the watch time per request as the immediate reward $r_t^u$. We sequentially conduct a series of experiments related to DCAF, RPAF (without PoolRank), and RPAF (with PoolRank). Each method is tested for 7 consecutive days, with the confidence intervals shown in Table \ref{table:live-exp-results}.

As shown in Table \ref{table:live-exp-results}, RPAF obtains a 1.13\% increase in daily watch time compared to the baseline and a 0.83\% increase compared to DCAF. We emphasize that in our system, a 0.1\% improvement holds statistical significance. Furthermore, a 150-day online experiment was conducted during which the experimental group employed DCAF, RPAF (without PoolRank) and RPAF (with PoolRank) in sequence, while the base group used the greedy strategy. Figure \ref{fig:dau} plots the comparison results of RPAF with the baseline, where the x-axis represents the number of days since deployment, and the y-axis indicates the percentage difference in the number of daily active users between the experimental group and the baseline group. Figure \ref{fig:dau} shows that the users in the experimental group are more likely to stay active on the platform, indicating an improvement in users' long-term values.

\begin{table}[t]
\caption{The online A/B test results.}
\centering
\begin{tabular}{c|c}
    \hline
    {Methods} & Daily WatchTime\\
    \hline
    baseline(greedy) &  - \\
    \hline
    DCAF & + 0.31\% \ \textcolor{gray}{[-0.12\%, 0.12\%]}\\
    \hline
    RPAF (without PoolRank) & + 0.91\% \ \textcolor{gray}{[-0.13\%, 0.13\%]}\\
    \hline
    \textbf{RPAF (with PoolRank)} & \textbf{+ 1.13\% \ \textcolor{gray}{[-0.11\%, 0.11\%]}} \\ \hline
\end{tabular}
\vspace{-2mm}
\label{table:live-exp-results}
\end{table}

\begin{figure}[t]
    \centering
    \includegraphics[width=0.9\linewidth,trim=0 12 0 0, clip]{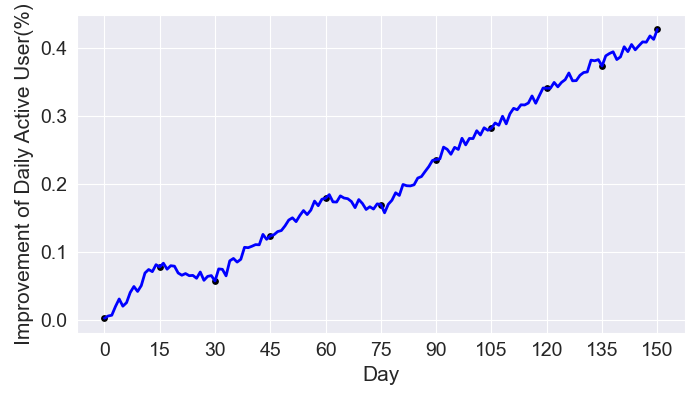}
    \caption{Results of DAU in online A/B experiments.}
    \label{fig:dau}
    \vspace{-3mm}
\end{figure}

\section{Conclusion}
This paper proposes RPAF, a cache allocation method that considers the value-strategy dependency and streaming allocation problem. RPAF is an RL-based two-stage framework containing prediction and allocation stages. The RLA is proposed to tackle the globality and strictness of budget constraints in the training of RPAF, and the PoolRank algorithm allocates the real-time and cached recommendations in a streaming manner. RPAF has proven its effectiveness in offline experiments and online A/B tests and has been deployed online in real-world applications, bringing a notable improvement.

\bibliographystyle{ACM-Reference-Format}
\bibliography{main}

\newpage

\onecolumn 

\appendix
\section{Proof of Proposition \ref{prop:trivial}} \label{appendix:proof:trivial}
We prove this by contradiction. Assume the optimal solution contains two users $u_1$ and $u_2$ with $c_t^{u_1}=c_t^{u_2}=1$, satisfying
\begin{equation}
a_t^{u_1} = 1, a_t^{u_2} = 0
\end{equation}
\begin{equation} \label{eq:proof:trivial:assumption-2}
\mathbb{E}\left[R_t^{u_1} |a_t^{u_1}=1\right] - \mathbb{E}\left[R_t^{u_1} |a_t^{u_1}=0\right] < \mathbb{E}\left[R_t^{u_2} |a_t^{u_2}=1\right] - \mathbb{E}\left[R_t^{u_2} |a_t^{u_2}=0\right]
\end{equation}

Then, we consider a new solution $a_t^{*u_1} = 0,a_t^{*u_2}=1$ with other cache choices unchanged. According to Eq. \eqref{eq:proof:trivial:assumption-2}, we have 
\begin{equation}
\begin{aligned}
&\sum_{u\in\mathcal{U}}c_t^u\mathbb{E}\left[R_t^u|a_t^{*u}\right] - \sum_{u\in\mathcal{U}}c_t^u\mathbb{E}\left[R_t^u|a_t^u\right] \\
=&\mathbb{E}\left[R_t^{u_1} |a_t^{*u_1}=0\right] + \mathbb{E}\left[R_t^{u_2} |a_t^{*u_2}=1\right] - \left\{\mathbb{E}\left[R_t^{u_1} |a_t^{u_1}=1\right] + \mathbb{E}\left[R_t^{u_2} |a_t^{u_2}=0\right]\right\} \\
=& \left\{\mathbb{E}\left[R_t^{u_2} |a_t^{*u_2}=1\right] - \mathbb{E}\left[R_t^{u_2} |a_t^{u_2}=0\right]\right\} - \left\{\mathbb{E}\left[R_t^{u_1} |a_t^{u_1}=1\right] - \mathbb{E}\left[R_t^{u_1} |a_t^{*u_1}=0\right]\right\}\\
>&0
\end{aligned}
\end{equation}
which means $a_t^{*u}$ is a better solution, leading to the contradiction.

\section{Proof of Proposition \ref{prop:consistency}} \label{appendix:proof:consistency}
According to Prop. \ref{prop:trivial}, the solution to Eq. \eqref{eq:strict-allocator} is
\begin{equation}
a_t^u = \textbf{arg-top}_M\left\{\left.Q\left(\boldsymbol{s}_t^u,1;\phi\right)-Q\left(\boldsymbol{s}_t^u,0;\phi\right)\right|u\in\mathcal{U},c_t^u=1\right\}
\end{equation}
Then, it suffices to prove that $Q\left(\boldsymbol{s}_t^u,1;\phi\right)-Q\left(\boldsymbol{s}_t^u,0;\phi\right)$ is monotonically increasing with regard to  $\tilde{\mu}\left(\boldsymbol{s}_t^u;\theta\right)$. Actually, according to the actor loss in Eq. \eqref{eq:cac-actor-sgd}, the optimal solution to $\tilde{\mu}\left(\boldsymbol{s}_t^u;\theta\right)$ satisfies $\partial L_a/\partial \tilde{\mu}=0$, i.e.
\begin{equation} \label{eq:minima-grad}
    \frac{\partial Q}{\partial \tilde{\mu}} = \alpha \frac{\partial T}{\partial \tilde{\mu}}
\end{equation}
Moreover, according to the definition of the critic function $Q$ with regard to the RLA in Eq. \eqref{eq:critic-function-tilde}, we have
\begin{equation}
    Q\left(\boldsymbol{s}_t^u,\tilde{\mu}\left(\boldsymbol{s}_t^u;\theta\right);\phi\right) = \tilde{\mu}\left(\boldsymbol{s}_t^u;\theta\right)Q\left(\boldsymbol{s}_t^u,1;\phi\right) + \left(1-\tilde{\mu}\left(\boldsymbol{s}_t^u;\theta\right)\right)Q\left(\boldsymbol{s}_t^u,0;\phi\right)
\end{equation}
Therefore we have
\begin{equation}
    \frac{\partial Q}{\partial \tilde{\mu}} = Q\left(\boldsymbol{s}_t^u,1;\phi\right) - Q\left(\boldsymbol{s}_t^u,0;\phi\right)
\end{equation}
and the optimal solution in Eq. \eqref{eq:minima-grad} becomes:
\begin{equation}
    \alpha \frac{\partial T}{\partial \tilde{\mu}} = Q\left(\boldsymbol{s}_t^u,1;\phi\right) - Q\left(\boldsymbol{s}_t^u,0;\phi\right)
\end{equation}
Moreover, since the penalty function $T$ is convex and $\alpha > 0$, $\alpha\partial T/\partial \tilde{\mu}$ is monotonically increasing with regard to $\tilde{\mu}$. Therefore, the larger $Q\left(\boldsymbol{s}_t^u,1;\phi\right) - Q\left(\boldsymbol{s}_t^u,0;\phi\right)$ is, the larger the solution $\tilde{\mu}$ is, which finishes the proof.

\section{\centering{Hyper Parameters}} \label{appendix:hyper-params}

\begin{table}[h]
\centering
\begin{tabular}{l|l}
    \hline
    Optimizer & Adam \\
    Actor Learning Rate & 0.0001 \\
    Critic Learning Rate & 0.0002 \\
    Number of Agents & 2 \\
    Action Dimensions Of Actor & 1 \\
    Discount Factor & 0.9 \\
    Action Upper Bound & 1.0 \\
    Replay Buffer Size & $1*10^{6}$ \\
    Train Batch Size & 1024 \\
    Fine-Tuning & True \\
    Normalized Observations & True \\
    Training Platform & Tensorflow \\ \hline
\end{tabular}
\caption{The Hyper-parameters of RPAF.}
\label{tab:hyper-params}
\end{table}

\end{document}